\documentclass[letterpaper, 12pt]{article}
\usepackage[papersize={8.5in,11in},top=1.1in,left=1.1in,right=1.1in,bottom=1.1in]{geometry}
\usepackage{amsmath, amssymb}
\usepackage{enumerate}
\allowdisplaybreaks
\usepackage{epsfig}
\usepackage{epstopdf}
\usepackage{cite}
\usepackage{multirow}
\usepackage{algorithm}
\usepackage{algorithmic}

\usepackage{tikz}
\usepackage{tikz,fullpage}
\usetikzlibrary{arrows,%
                petri,%
                topaths}%
\usepackage{tkz-berge}
\usepackage[position=top]{subfig}

\newcommand{\su}{\succ}

\renewcommand{\vec}[1]{\mathbf{#1}}
\DeclareMathOperator{\dist}{dist}
\renewcommand{\S}{\mathcal{S}}

\DeclareMathOperator{\cost}{SC}
\DeclareMathOperator{\med}{med}

\newcommand\ddfrac[2]{\frac{\displaystyle #1}{\displaystyle #2}}

\newtheorem{thm}{Theorem}
\newtheorem{lem}[thm]{Lemma}

\title{Randomized Social Choice Functions\\ Under Metric Preferences}

\author{Elliot Anshelevich \\ John Postl \\ Rensselaer Polytechnic Institute\\ 110 8th Street, Troy, NY 12180\\ eanshel@cs.rpi.edu, postlj@rpi.edu}

\usetikzlibrary{positioning}
\tikzset{main node/.style={circle,fill=blue!20,draw,minimum size=1cm,inner sep=0pt},
            }

\begin{document}

\maketitle

\begin{abstract}
We determine the quality of randomized social choice mechanisms in a setting in which the agents have \emph{metric preferences}: every agent has a cost for each alternative, and these costs form a metric. We assume that these costs are unknown to the mechanisms (and possibly even to the agents themselves), which means we cannot simply select the optimal alternative, i.e. the alternative that minimizes the total agent cost (or median agent cost). However, we do assume that the agents know their ordinal preferences that are induced by the metric space. We examine randomized social choice functions that require only this ordinal information and select an alternative that is good in expectation with respect to the costs from the metric. To quantify how good a randomized social choice function is, we bound the \emph{distortion}, which is the worst-case ratio between expected cost of the alternative selected and the cost of the optimal alternative. We provide new distortion bounds for a variety of randomized mechanisms, for both general metrics and for important special cases. Our results show a sizable improvement in distortion over deterministic mechanisms.
\end{abstract}

\section{Introduction}
Social choice, and especially the recent field of {\em computational social choice}, is a large and exciting subfield of artificial intelligence research (see for example \cite{chevaleyre2007short,moulin2016handbook} for some surveys and connections with other areas of AI). The goal of social choice theory is usually to aggregate the preferences of many agents with conflicting interests, and produce an outcome that is suitable to the whole rather than to any particular agent. This is accomplished via a \emph{social choice mechanism} which maps the preferences of the agents, usually represented as total orders over the set of alternatives, to a single winning alternative. There is no agreed upon ``best'' social choice mechanism; it is not obvious how one can even make this determination. Because of this, much of social choice literature is concerned with defining normative or axiomatic criteria, so that a social choice mechanism is ``good'' if it satisfies many useful criteria.

Another method of determining the quality of a social choice function is the \emph{utilitarian} approach, which is often used in welfare economics and algorithmic mechanism design. Here agents have an associated utility (or cost, as in this paper) with each alternative that is a measure of how desirable (or undesirable) an alternative is to an agent. We can define the quality of an alternative to be a function of these agent utilities, for example as the sum of all agent utilities for a particular alternative. Other objective functions such as the median or max utility of the agents for a fixed alternative can be used as well. The utilitarian approach has received a lot of attention recently in the social choice literature \cite{caragiannis2011voting,filos2014social,harsanyi1976cardinal,feldman2016}, see especially \cite{boutilier2015optimal} for a thorough discussion of this approach, its strengths, and its weaknesses.


A frequent criticism of the utilitarian approach is that it is unreasonable to assume that the mechanism, or even the agents themselves, know what their utilities are. Indeed, it can be difficult for an agent to quantify the desirability of an alternative into a single number, but there are arguments in favor of cardinal utilities \cite{boutilier2015optimal,harsanyi1976cardinal}. Even if the agents were capable of doing this for each alternative, it could be difficult for us to elicit these utilities in order to compute the optimal alternative. It is much more reasonable, and much more common, to assume that the agents know the preference rankings induced by their utilities over the alternatives. That is, it might be difficult for an agent to express exactly how she feels about alternatives $X$ and $Y$, but she should know if she prefers $X$ to $Y$. Because of this, work such as \cite{procaccia2006distortion,boutilier2015optimal,caragiannis2011voting,anshelevich2014approximating,feldman2016} considers how well social choice mechanisms can perform when they only have access to {\em ordinal preferences} of the agents, i.e., their rankings over the alternatives, instead of the true underlying (possibly latent) utilities. The \emph{distortion} of a social choice function is defined here as the worst-case ratio of the cost of the alternative selected by the social choice function and the cost of the truly optimal alternative.

Our goal in this work is to design social choice mechanisms that minimize the worst-case distortion for the sum and median objective functions when the agents have \emph{metric preferences} \cite{anshelevich2014approximating}. That is, we assume that the costs of agents over alternatives form an arbitrary \emph{metric space} and that their preferences are induced by this metric space. Assuming such metric or spatial preferences is common \cite{enelow1984spatial}, has a natural interpretation of agents liking candidates/alternatives which are most similar to them, such as in facility location literature \cite{campos2008multiple,escoffier2011strategy,feldman2016}, and our setting is sufficiently general that it does not impose any restrictions on the set of allowable preference profiles. Anshelevich et al. \cite{anshelevich2014approximating} provide distortion bounds for this setting using well-known deterministic mechanisms such as plurality, Copeland, and ranked pairs. We improve on these results by providing distortion guarantees for \emph{randomized social choice functions}, which output a probability distribution over the set of alternatives rather than a single winning alternative.
We show that our randomized mechanisms perform better than any deterministic mechanism, and provide optimal randomized mechanisms for various settings.

We also examine the distortion of randomized mechanisms in important specialized settings. Many of our worst-case examples occur when many agents are indifferent between their top alternative and the optimal alternative. In many settings, however, agents are more \emph{decisive} about their top choice, and prefer it much more than any other alternative. We introduce a formal notion of \emph{decisiveness}, which is a measure of how strongly an agent feels about her top preference relative to her second choice. If an agent is very decisive, then she is very close to her top choice compared to her second choice in the metric space. In the extreme case, this means that the set of agents and alternatives is identical \cite{goel2012triadic}, as can occur for example when proposal writers rank all the other proposals being submitted, or when a committee must choose one of its members to lead it.
We demonstrate that when agents are decisive, the distortion greatly improves, and quantify the relation between decisiveness and the performance of social choice mechanisms. Finally, we consider other natural special cases, such as when preferences are 1-Euclidean and when alternatives are vertices of a simplex. 1-Euclidean preferences are already recognized as a well-studied and well-motivated special case \cite{elkind2014recognizing,procaccia2013approximate}. The setting in which alternatives form a simplex corresponds to the case in which alternatives share no similarities, i.e., when all alternatives are equally different from each other.


\subsection{Our Contributions}

In this paper, we bound the worst-case distortion of several randomized social choice functions in many different settings. Recall that the distortion is the worst-case ratio of the expected value of the alternative selected by the randomized mechanism and the optimal alternative. We use two different objective functions for the purpose of defining the quality of an alternative. The first is the sum objective, which defines the social cost of an alternative to be the sum of agent costs for that particular alternative. We also consider the median objective, which defines the quality of an alternative as the median agent's cost for that alternative.

We summarize our results in Table~\ref{table:results_summary}. Note that for the sum objective, these results are also given for $\alpha$-decisive metric spaces. A metric space is $\alpha$-decisive if for every agent, the cost of her first choice is less than $\alpha$ times the cost of her second choice, for some $\alpha \in [0,1]$. In other words, this provides a constraint on how indifferent an agent can be between her first and second choice. By definition, any agent cost function is 1-decisive. Considering $\alpha$-decisive metric spaces allows us to immediately give results for important subcases, such as $0$-decisive metric spaces in which every agent has distance 0 to her top alternative, i.e., every agent is also an alternative.

For the sum objective function, we begin by giving a lower bound of $1+\alpha$ for \emph{all} randomized mechanisms, which corresponds to a lower bound of 2 for general metric spaces. This is smaller than the lower bound of 3 for deterministic mechanisms from \cite{anshelevich2014approximating}. One of our first results is to show randomized dictatorship has worst-case distortion strictly better than 3, which is better than any possible deterministic mechanism. Furthermore, we show that a generalization of the ``proportional to squares" mechanism is the optimal randomized mechanism when there are two alternatives, i.e., it has a distortion of $1+\alpha$.

We also examine how well randomized mechanisms perform in important subcases. We consider the well-known case in which all agents and alternatives are points on a line with the Euclidean metric, known as 1-Euclidean preferences \cite{elkind2014recognizing}. We give an algorithm, which heavily relies on proportional to squares, to achieve the optimal distortion bound of $1+\alpha$ for any number of alternatives. We also consider a case first briefly described in \cite{anshelevich2014approximating}, known as the $(m-1)$-simplex setting, in which the alternatives are vertices of a simplex and the agents lie in the simplex. This corresponds to alternatives sharing no similarities. We are able to show that proportional to squares achieves worst-case distortion of $\frac{1}{2}\left(1 + \alpha + \sqrt{2}\sqrt{\alpha^2+1} \right)$, which is fairly close to the optimal bound of $1+\alpha$. For details, see Section \ref{sec:simplex}. 

\begin{table}[]
\centering

\begin{tabular}{|l|l|l|l|}
\hline
            & \multicolumn{2}{c|}{Sum}                                                                                                 & \multicolumn{1}{c|}{Median}                                        \\ \hline
            & General                   & $\alpha$-Decisive                                                                          & General                                                            \\ \hline
General     & {\em RD:} $3 - \frac{2}{n}$ & $2 + \alpha - \frac{2}{n}$                                                             & \begin{tabular}[c]{@{}l@{}}{\em Uncovered Set Min-Cover:} 4\end{tabular} \\ \hline
1-Euclidean & {\em 1-D Prop. to Squares:} 2   & $1+\alpha$                                                                                 & {\em Condorcet:} 3                                                       \\ \hline
Simplex     & {\em Prop. to Squares:} 2       & $\frac{1}{2}\left(1 + \alpha + \sqrt{2}\sqrt{\alpha^2+1} \right)$ & {\em Majority consistent:} 2                                             \\ \hline
Lower Bounds& 2              &  $1+\alpha$    & {\em 1-Euclidean:} 3 \\ \hline
\end{tabular}
\caption{The worst-case distortion of our social choice mechanisms are given for both the sum and median objective functions in various settings.
In the general setting, all randomized mechanisms have a lower bound of 2 and 3 for the sum and median objective functions, respectively. For the $\alpha$-decisive setting with the sum objective function, no randomized mechanism can have distortion better than $1+\alpha$.}
\label{table:results_summary}
\end{table}


Our other major contribution is defining a new randomized mechanism for the median objective which achieves a distortion of 4 in arbitrary metric spaces (we call this mechanism {\em Uncovered Set Min-Cover}). This requires forming a very specific distribution over all alternatives in the uncovered set, and then showing that this distribution ensures that no alternative ``covers" more than half of the total probability of all alternatives. We do this by taking advantage of LP-duality combined with properties of the uncovered set. We believe that this mechanism is interesting on its own, as it is likely to have other nice properties in addition to low median distortion.

\subsection{Related Work}

Embedding the unknown cardinal preferences of agents into an ordinal space and measuring the \emph{distortion} of social choice functions that operate on these ordinal preferences was first done in \cite{procaccia2006distortion}. Additional papers \cite{boutilier2015optimal,caragiannis2011voting,oren2014online,anshelevich2014approximating,feldman2016} have since studied distortion and other related concepts of many different mechanisms with various assumptions about the utilities/costs of the agents. In this context, Anshelevich et al. introduced the notion of \emph{metric preferences} in \cite{anshelevich2014approximating}, which assumes the costs of the agents and alternatives form a metric. For this setting, Anshelevich et al. \cite{anshelevich2014approximating} proved that while various scoring rules such as Plurality and Borda can have very large distortion, the Copeland social choice function always has distortion at most 5, and in fact no deterministic social choice function can have worst-case distortion better than 3. For the median distortion objective, they proved that Copeland still achieves distortion of 5, and in fact no deterministic function can have worst-case distortion better than 5; thus in terms of worst-case distortion Copeland is optimal for this objective. We further extend their work by considering randomized mechanisms instead of deterministic ones and exploring special types of metrics. The randomized mechanisms we provide have smaller expected distortion than the deterministic mechanisms from \cite{anshelevich2014approximating}, and in fact sometimes perform better than any deterministic mechanism possibly could.

Using mechanisms to select alternatives from a metric space when the true locations of agents is unknown is also reminiscent of facility location games \cite{campos2008multiple,escoffier2011strategy}. However, we select only a single winning alternative in our setting, while in these papers, they select multiple facilities.

Pivato \cite{pivato2016asymptotic} demonstrates that social choice functions like Borda and approval voting are able to maximize the utility with high probability, when the agents satisfy certain properties. Rivest and Shen \cite{rivest2010optimal} use a game-theoretic model to compare two voting systems and develop a randomized mechanism that is always preferred to any other voting system.

Assuming that the preferences of agents are induced by a metric is a type of spatial preference \cite{enelow1984spatial,merrill1999unified}. There are many other notions of spatial preferences that are prevalent in social choice, such as 1-Euclidean preferences \cite{elkind2014recognizing,procaccia2013approximate}, single-peaked preferences \cite{sui2013multi}, and single-crossing \cite{gans1996majority}. We consider 1-Euclidean preferences as an important special case of the metric preferences we study in this paper.

Randomized social choice was first studied in \cite{Zeckhauser01111969,FISHBURN1972189,10.2307/2296588}. A similar setting was considered by Fishburn and Gehrlein \cite{10.2307/1914118}, in which agents are uncertain about their preferences and express their preferences using probability distributions. We consider several randomized mechanisms, such as randomized dictatorship \cite{chatterji2014random}; other randomized voting mechanisms have been used in, e.g., \cite{procaccia2010,Brandl2016}. The use of randomized mechanisms is seen very frequently in literature concerning one-sided matchings. Random serial dictatorship and probabilistic serial are perhaps the most well-studied randomized mechanisms, and there is a significant amount of literature on them (e.g. \cite{Bogomolnaia2001295,aziz2013tradeoff,aziz2014generalization,chakrabarty2014welfare,christodoulou2016welfare,filos2014social}).  In particular, the results of \cite{filos2014social,greedy} are analogous to finding the distortion of matching mechanisms.

Related to the notion of randomized social choice functions are proportional representation voting systems in which there are multiple winners \cite{monroe1995fully,nandeibam2003distribution,skowron2015fully,Chamberlin}. Selecting multiple winners is conceptually similar to having a probability distribution over a set of alternatives. Skowron et al. \cite{skowron2015fully} consider approximation algorithms to multiwinner rules that seek to maximize global objective functions, but are NP-hard to solve.

Finally, independently from us, Feldman et al. have also recently considered the distortion of randomized social choice functions in \cite{feldman2016}. While they mostly focus on truthful mechanisms (i.e., the "strategic" setting), there is some intersection between our results. Specifically, Feldman et al. also give a bound of 3 (and a lower bound of 2) for arbitrary metric spaces in the sum objective, and also provides a mechanism with distortion 2 for the 1-Euclidean case. The latter mechanism is quite different from ours, however: ours seems to be somewhat simpler, but the mechanism from \cite{feldman2016} has the advantage of being truthful. However, Feldman et al. do not consider either $\alpha$-decisive voters or the median objective: showing better performance for decisive voters and designing better mechanisms for the median objective are two of our major contributions.

\section{Preliminaries}\label{sec:prelim}

\paragraph{Social Choice with Ordinal Preferences.} Let $N = \{1, 2, \ldots, n \}$ be the set of agents, and let $M = \{A_{1}, A_{2}, \ldots, A_{m} \}$ be the set of alternatives. Let $\S$ be the set of all total orders on the set of alternatives $M$. We will typically use $i, j$ to refer to agents and $W, X, Y, Z$ to refer to alternatives. Every agent $i \in N$ has a \emph{preference ranking }$\sigma_{i} \in \S$; by $X \succ_{i} Y$ we will mean that $X$ is preferred over $Y$ in ranking $\sigma_i$. We call the vector $\sigma = (\sigma_{1}, \ldots, \sigma_{n}) \in \S^{n}$ a \emph{preference profile}. We say that an alternative $X$ \emph{pairwise defeats} $Y$ if $|\{i \in N: X \succ_{i} Y \} | > \frac{n}{2}$. Furthermore, we use the following notation to describe sets of agents with particular preferences: $XY = \{i \in N: X \su_{i} Y \}$ and $X^* = \{i \in N: X \su_{i} Y \text{ for all } Y \neq X\}$.

Once we are given a preference profile, we want to aggregate the preferences of the agents and select a single alternative as the winner or find a probability distribution over the alternatives and pick a single winner according to that distribution. A \emph{deterministic social choice function} $f: \S^n \to M$ is a mapping from the set of preference profiles to the set of alternatives. A \emph{randomized social choice function} $f: \S^n \to \Delta(M)$ is a mapping from the set of preference profiles to the space of all probability distributions over the alternatives $\Delta(M)$. Some well-known social choice functions which we consider in this paper are as follows.
\begin{itemize}
\item{\textbf{Randomized dictatorship/plurality:}} The winning alternative is selected according to the following probability distribution: for all alternatives $Y \in M$, $$p(Y) = \frac{|Y^*|}{n}.$$
\item{\textbf{Proportional to squares.}} The winning alternative is selected according to the following probability distribution: for all alternatives $Y \in M$, $$p(Y) = \frac{|Y^*|^2}{\sum_{Z \in M} |Z^*|^2}.$$
\item{\textbf{Condorcet method:}} A weak Condorcet winner is defined as the alternative that either pairwise defeats or pairwise ties every other alternative. There can be multiple weak Condorcet winners. A Condorcet winner must pairwise defeat every other alternative; there can be at most one Condorcet winner. Neither weak Condorcet winners nor Condorcet winners are guaranteed to exist. A Condorcet method is any social choice function that is guaranteed to select a Condorcet winner, if it exists.
\item{\textbf{Majority method:}} A majority winner is an alternative that is ranked as the first preference of strictly more than $\frac{n}{2}$ agents. A majority method is any method that will select the majority winner, if it exists.
\end{itemize}

\paragraph{Cardinal Metric Costs.} In our work we take the utilitarian view, and study the case when the ordinal preferences $\sigma$ are in fact a result of the underlying cardinal agent costs. Formally, we assume that there exists an arbitrary metric $d: (N \cup M)^2 \to \mathbb{R}_{\geq 0}$ on the set of agents and alternatives (or more generally a {\em pseudo-metric}, since we allow distinct agents and alternatives to be identical and have distance 0). Here $d(i,X)$ is the cost incurred by agent $i$ when alternative $X$ is selected as the winner; these costs can be arbitrary but are assumed to obey the triangle inequality. The metric costs $d$ naturally give rise to a preference profile. Formally, we say that $\sigma$ is \emph{consistent} with $d$ if $\forall i \in N, \forall X, Y \in M$, if $d(i, X) < d(i, Y)$, then $X \succ_{i} Y$. In other words, if the cost of $X$ is less than the cost of $Y$ for an agent, then the agent should prefer $X$ over $Y$. When $d(i,X)=d(i,Y)$, then both $X \succ_{i} Y$ and $Y \succ_{i} X$ are considered consistent with the costs of $i$. Let $\rho(d)$ denote the set of preference profiles consistent with $d$ ($\rho(d)$ may include several preference profiles if the agent costs have ties). Similarly, we define $\rho^{-1}(\sigma)$ to be the set of metrics such that $\sigma \in \rho(d)$.




\paragraph{Social Cost and Distortion.} We measure the quality of each alternative using the costs incurred by all the agents when this alternative is chosen. We use two different notions of social cost. First, we study the sum objective function, which is defined as $\cost(X, d) =  \sum_{i \in N} d(i, X)$ for an alternative $X$. We also study the median objective function, $\med(X, d) =  \med_{i \in N}(d(i, X))$.  Since we have defined the cost of alternatives, we can now give the cost of an outcome of a deterministic social choice function $f$ as $\cost(f(\sigma),d)$ or $\med(f(\sigma), d)$. For randomized functions, we define the cost of an outcome, which is a probability distribution over alternatives, as follows: $\cost(f(\sigma),d) = \mathbb{E}_{X\sim f(\sigma)}\left[\cost(X,d) \right] = \sum_{X \in M} p(X) \cost(X,d)$ and $\med(f(\sigma),d) = \mathbb{E}_{X\sim f(\sigma)}\left[\med(X,d) \right] = \sum_{X \in M} p(X) \med(X,d)$, where $p(X)$ is the probability of alternative $X$ being selected, according to $f(\sigma)$. When the metric $d$ is obvious from context, we will use $\cost(X)$ and $\med(X)$ as shorthand.


As described in the Introduction, we can view social choice mechanisms in our setting as attempting to find the optimal alternative (one that minimizes cost), but only having access to the ordinal preference profile $\sigma$, instead of the full underlying costs $d$. Since it is impossible to compute the optimal alternative using only ordinal preferences, we would like to determine how well the aforementioned social choice functions select alternatives based on their social costs, despite only being given the preference profiles. In particular, we would like to quantify how the social choice functions perform in the worst-case. To do this, we use the notion of \emph{distortion} from \cite{procaccia2006distortion,boutilier2015optimal}, defined as follows.
\begin{align*}
\dist_{\sum}(f, \sigma) &= \sup_{d \in \rho^{-1}(\sigma)} \frac{ \cost(f(\sigma), d)}{\min_{X \in M} \cost(X, d)} \\
\dist_{\med}(f, \sigma) &=  \sup_{d \in \rho^{-1}(\sigma)} \frac{ \med(f(\sigma), d)}{\min_{X \in M} \med(X, d)}.
\end{align*}

In other words, the distortion of a social choice mechanism $f$ on a profile $\sigma$ is the worst-case ratio between the social cost of $f(\sigma)$, and the social cost of the true optimum alternative. The worst-case is taken over all metrics $d$ which may have induced $\sigma$, since the social choice function does not and cannot know which of these metrics is the true one.

\paragraph{Examples.} To illustrate some of the behavior arising in our setting, and to build intuition, here we consider a simple example. Consider the setting in Figure~\ref{fig:sumLowerbound} with only two alternatives, $X$ and $Y$. The preferences are tied: $\frac{n}{2}$ agents prefer $X$ to $Y$, and $\frac{n}{2}$ prefer $Y$ to $X$. The ordinal social choice functions we consider do not know anything else; a deterministic function would be forced to choose a specific alternative (without loss of generality suppose it is $Y$), while randomized dictatorship would choose each alternative with probability $\frac{1}{2}$. The true, underlying costs could be as follows, however: $\frac{n}{2}$ agents have cost $0$ for $X$ and $2$ for $Y$ (these are located ``on top of" $X$), while $\frac{n}{2}$ agents have cost $1+\epsilon$ for $X$ and $1-\epsilon$ for $Y$, for some very small $\epsilon$ (these are located ``between $X$ and $Y$"). Then $X$ is the true optimum solution: the total social cost of $X$ is $(1+\epsilon)\frac{n}{2}$, while the social cost of $Y$ is $(3-\epsilon)\frac{n}{2}$. Thus, any deterministic function selecting $Y$ has (sum) distortion approaching 3 as $\epsilon\rightarrow 0$, while randomized dictatorship has expected distortion approaching $\frac{1}{2}\cdot1+\frac{1}{2}\cdot 3=2$ for this example.

\begin{figure}
\centering
\begin{tikzpicture}[scale=1,transform shape]
        \Vertex[x=0,y=0,LabelOut=true,L=$\frac{n}{2}$ agents,Lpos=270]{j}
  \Vertex[x=0,y=0,]{X}

  \Vertex[x=8,y=0]{Y}
    \Vertex[x=4,y=0,LabelOut=true,L=$\frac{n}{2}$ agents,Lpos=270]{i}

  \tikzstyle{LabelStyle}=[fill=white,sloped]

  \Edge[label=$(1+\epsilon)$](X)(i)
  \Edge[label=$(1-\epsilon)$](i)(Y)
 \tikzstyle{EdgeStyle}=[bend left]

  \Edge[label=$2$](X)(Y)
  \tikzstyle{EdgeStyle}=[bend right]
\end{tikzpicture}
\caption{There are $\frac{n}{2}$ agents located at $X$ who prefer $X$ and $\frac{n}{2}$ agents between $X$ and $Y$ who prefer $Y$. As $\epsilon \rightarrow 0$, the expected distortion of randomized dictatorship approaches 2.}
\label{fig:sumLowerbound}
\end{figure}
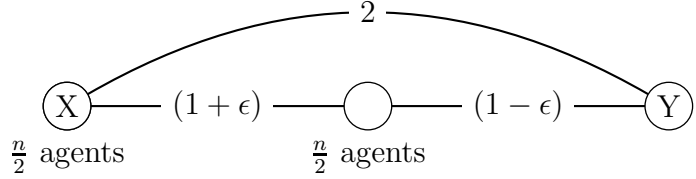

For the median objective, suppose instead that there is an odd number of voters, with $\lceil\frac{n}{2}\rceil$ preferring $Y$ and $\lfloor\frac{n}{2}\rfloor$ preferring $X$, as seen in Figure~\ref{fig:medianExample}. Any reasonable social choice function would select $Y$; randomized dictatorship would once again mix about equally between $X$ and $Y$. However, the true numerical costs can be as follows: $\lfloor\frac{n}{2}\rfloor$ have cost $0$ for $X$ and $2$ for $Y$, one agent has cost $1-\epsilon$ for $Y$ and $1+\epsilon$ for $X$, and $\lfloor\frac{n}{2}\rfloor$ have cost of $2$ for $Y$ and $4$ for $X$. The median agent cost for $X$ is approximately 1, while the median agent cost for $Y$ is 2. Thus, $X$ is the optimum solution, but random dictatorship only chooses it with probability about $\frac{1}{2}$. For more examples and lower bounds on possible distortion, see Theorems \ref{thm:randomizedLowerBound} and \ref{thm:median_random_lb}.

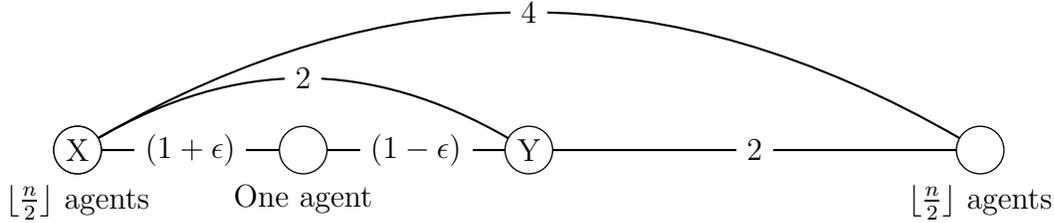
\begin{figure}
\begin{tikzpicture}[scale=1,transform shape]
  \Vertex[x=12,y=0,LabelOut=true,Lpos=270,L=$\lfloor\frac{n}{2}\rfloor$ agents]{k}
  \Vertex[x=0,y=0,LabelOut=true,L=$\lfloor\frac{n}{2}\rfloor$ agents,Lpos=270]{j}
  \Vertex[x=0,y=0,]{X}
  \Vertex[x=6,y=0]{Y}
    \Vertex[x=3,y=0,LabelOut=true,L=One agent,Lpos=270]{i}

  \tikzstyle{LabelStyle}=[fill=white,sloped]

  \Edge[label=$(1+\epsilon)$](X)(i)
  \Edge[label=$(1-\epsilon)$](i)(Y)
    \Edge[label=$2$](Y)(k)
 \tikzstyle{EdgeStyle}=[bend left]
   \Edge[label=$4$](X)(k)
  \Edge[label=$2$](X)(Y)
\end{tikzpicture}
\caption{There are $\lfloor\frac{n}{2}\rfloor$ agents located at $X$ who prefer $X$, one agent between $X$ and $Y$ that prefers $Y$, and $\lfloor\frac{n}{2}\rfloor$ agents far from either alternative that prefer $Y$. As $\epsilon \rightarrow 0$, the expected distortion of randomized dictatorship is $\frac{3}{2}$.}
\label{fig:medianExample}
\end{figure}

\paragraph{Decisive Voters.} Many of our worst-case examples occur when many agents are indifferent between their top alternative and the optimal alternative. In many settings, however, agents are more \emph{decisive} about their top choice, and prefer it much more than any other alternative. Formally, we say that an agent $i$ whose top choice is $W$ and second choice is $X$ is $\alpha$-\emph{decisive} if $d(i, W) \leq \alpha \cdot d(i, X)$ where $\alpha \in [0, 1]$. We say that a metric space is $\alpha$-\emph{decisive} if for some fixed $\alpha$, every agent is $\alpha$-\emph{decisive}. Every metric space is $1$-decisive, while a metric space in which every agent has distance 0 to her top alternative is $0$-decisive. In fact, $0$-decisive metric spaces are interesting in their own right: they include the case when each voter must exactly coincide with some alternative, and so capture the settings where the set of voters and alternatives is the same. This occurs when every voter corresponds to a possible alternative, such as when a committee must vote to choose one of its members to lead it, or when writers of NSF proposals vote for each others' proposals to be funded.


Note that when talking about $\alpha$-decisive metrics, $\rho^{-1}(\sigma)$ denotes the set of all $\alpha$-decisive metrics $d$ such that $\sigma$ is consistent with them (as opposed to the set of all such possible metrics). Thus, when we consider distortion in the $\alpha$-decisive setting, it measures the quality of an algorithm with only ordinal knowledge, as compared to the quality of the true optimum solution, assuming that the underlying metric is $\alpha$-decisive.

\section{Distortion of the Sum of Agent Costs}
\subsection{General Metric Spaces}

In this section, we examine the sum objective and provide mechanisms with low distortion.
We first show that for general metric spaces, the randomized dictatorship mechanism has a distortion of less than 3, which is better than any deterministic mechanism, since all deterministic mechanisms have a worst-case distortion of at least 3 \cite{anshelevich2014approximating}. We then consider the case of two alternatives, and give the best possible randomized mechanism for this special case. As it is more general, we consider the $\alpha$-decisive setting: results for arbitrary metric spaces are simply the results for $1$-decisive agents. In all of our results, we observe that the worst-case distortion is linearly dependent on $\alpha$: the more decisive agents are, the better our mechanisms are able to perform.

We begin this section by addressing the question of how well \emph{any} randomized social choice function can perform. Our first theorem shows that no randomized mechanism can find an alternative that is in expectation within a factor strictly smaller than $1+\alpha$ from the optimum alternative for $\alpha$-decisive metric spaces. Thus no mechanism can have distortion better than 2 for general metric spaces. In comparison, the best known distortion lower bound for deterministic mechanisms is equal to 3 (from \cite{anshelevich2014approximating}).


\begin{thm}\label{thm:randomizedLowerBound}
The worst-case distortion of any randomized mechanism when the metric space is $\alpha$-decisive is at least $1+\alpha$.
\end{thm}
\begin{proof}
We must show that there exists a preference profile such that for all randomized mechanisms, there always exists an $\alpha$-decisive metric space that induces the preference profile and where the distortion is at least $1+\alpha$. 
We will consider a preference profile with $m=2$ alternatives $W,X$ and $n$ agents ($n$ is even) where $\frac{n}{2}$ agents prefer $W$ over $X$ and $\frac{n}{2}$ agents prefer $X$ over $W$.  We claim that no randomized mechanism can have distortion $< 1 +\alpha$ for all metric spaces that induce this profile.

First, we will consider an $\alpha$-decisive metric space that induces the preference profile and where $X$ is optimal. All agents $i$ who prefer $X$ have $d(i, X) = 0$ and $d(i, W) = 1$. The remaining agents have $d(i, W) = \frac{\alpha}{1+\alpha}, d(i, X) = \frac{1}{1+\alpha}$. Thus, $\cost(X) = \frac{n}{2} \cdot \frac{1}{1+\alpha}$ and $\cost(W) = \frac{n}{2}(\frac{\alpha}{1+\alpha} + 1 )$. The distortion of selecting alternative $W$ is $\frac{\cost(W)}{\cost(X)} = 1+2\alpha$. Obviously the distortion of selecting the optimal alternative $X$ is 1. Thus, for any randomized mechanism, the distortion is $p(X) + p(W)(1 + 2\alpha)$, where $p(X), p(W)$ are the probabilities of the randomized mechanism selecting $X$ and $W$, respectively.

Next, we claim there exists a similar $\alpha$-decisive metric space that induces the preference profile and where $W$ is optimal in which the distortion is $p(W) + p(X)(1+2\alpha)$.

Since the mechanism does not know the metric space (or which of $X, W$ is optimal), it cannot obtain a worst-case distortion better than $\max(p(X) + p(W)(1 + 2\alpha), p(W) + p(X)(1+2\alpha))$ since either metric space could have induced the preference profile. Clearly, the worst-case distortion is minimized when $$p(X) + p(W)(1 + 2\alpha) = p(W) + p(X)(1+2\alpha). $$ This reduces to $p(W) = p(X) = \frac{1}{2}$. We observe that $p(X) + p(W)\left(1+2\alpha \right) = 1 + \alpha$ in this case, which gives us the desired lower bound.
\end{proof}

We will now prove several helpful lemmas that are necessary in order to upper-bound the worst-case distortion of our randomized social choice mechanisms. Our first lemma provides a refinement over the standard bound of $d(i,W) \geq \frac{1}{2} d(W,Y)$ (from \cite{anshelevich2014approximating}) for agents that prefer $Y$ to $W$ when the agents are in $\alpha$-decisive spaces and $Y$ is their first preference as well. As we will see, this latter requirement does not impede our ability to find better lower bounds for the optimal alternative in $\alpha$-decisive metrics.

\begin{lem}
\label{lem:dec_agent_lb}
If a metric space is $\alpha$-decisive, then for all alternatives $W \neq Y$, $d(i, W) \geq \frac{1}{1 + \alpha} \cdot d(W, Y)$, for every voter $i \in Y^*$.
\end{lem}
\begin{proof}
Consider an $\alpha$-decisive agent $i$ with top choice $Y$ and second choice $Z$. $W$ is an alternative, different from $Y$. By definition, $d(i,Y) \leq \alpha \cdot d(i, Z) \leq   \alpha \cdot d(i, W)$. We observe that
\begin{align*}
d(i,W) &\geq d(W, Y) - d(i, Y) \\
&\geq d(W, Y) - \alpha \cdot d(i, W),
\end{align*}
which implies that $d(i, W) \geq \frac{1}{1+ \alpha} d(W, Y)$.
\end{proof}

We can now derive an improved lower bound of the social cost of the optimal alternative $X$. This is done by applying Lemma~\ref{lem:dec_agent_lb} to every agent (and with alternative $W$ in the lemma being set to $X$) and summing the resulting inequalities.

\begin{lem}
\label{lem:dec_alternative_lb}
If a metric space is $\alpha$-decisive, then for any alternative $X \in M$, $\cost(X) \geq \frac{1}{1 + \alpha} \sum_{Y \in M} |Y^*|  \cdot d(X, Y)$.
\end{lem}

Our next lemma is the first pertaining to upper-bounding the worst-case distortion of randomized social choice functions. This lemma parameterizes the distortion by the probability distribution over the alternatives. Thus, it is easily used to quickly bound the distortion for several randomized social choice functions by simply plugging in the appropriate probabilities for each alternative $Y$.

\begin{lem}
\label{lem:dist_bound_dec}
For any instance $\sigma$, social choice function $f$, and $\alpha$-decisive metric space, $$\dist_{\sum}(f, \sigma) \leq 1  + \frac{(1+\alpha)\sum_{Y \in M}p(Y)(n - \frac{2}{1+\alpha}|Y^*|)d(X,Y)}{\sum_{Y \in M} |Y^*| d(X, Y)},$$ where $X$ is the optimal alternative and $p(Y)$ is the probability that alternative $Y$ is selected by
$f$ given profile $\sigma$.
\end{lem}
\begin{proof} Consider an alternative $Y \neq X$: we want to upper-bound $\cost(Y)$. For all $i \in Y^*$, we know that $d(i,Y) \leq \alpha\cdot d(i,X)$ by the definition of $\alpha$-decisiveness. More generally, for $i \in YX$, we have a weaker bound of $d(i,Y) \leq d(i,X)$. Finally, for $i \in XY$, we can use the triangle inequality to obtain $d(i,Y) \leq d(i,X) + d(X,Y)$. Combining these three inequalities together, we are able to derive
\begin{align*}\cost(Y)
&= \sum_{i \in N} d(i,Y) \\
&\leq  \alpha \sum_{i \in Y^*} d(i,X) + \sum_{i \in YX \setminus  Y^*} d(i,X) + \sum_{i \in XY} \left(d(i, X) +d(X,Y) \right) \\
&= \sum_{i \in N} d(i,X) +  |XY| d(X,Y)  - (1-\alpha)\sum_{i \in Y^*} d(i, X) \\
&= \sum_{i \in N} d(i,X) +  (n - |YX|) d(X,Y)  - (1-\alpha)\sum_{i \in Y^*} d(i, X).
\end{align*}
We know that $|YX| \geq |Y^*|$. Furthermore, by Lemma~\ref{lem:dec_agent_lb}, we know that for $i \in Y^*$, $d(i,X) \geq \frac{1}{1+\alpha} d(X,Y)$. We can apply these two bounds to our previous expression to conclude that
\begin{align}
\cost(Y)
&\leq  \sum_{i \in N} d(i,X) +  (n - |Y^*|) d(X,Y)  - \frac{1-\alpha}{1+\alpha} |Y^*|d(X,Y) \\
&=  \cost(X) + \left(n - \frac{2}{1+\alpha}|Y^*|\right) d(X,Y).\label{eq.1}
\end{align}

In addition to an upper bound for $\cost(Y)$ where $Y \neq X$, we need a lower bound for the cost of the optimal alternative $X$. By Lemma~\ref{lem:dec_alternative_lb}, we have that
\begin{equation}\label{eq.2}
\cost(X) \geq \frac{1}{1+\alpha} \sum_{Y \in M} |Y^*| \cdot d(X,Y).
\end{equation}
With these two inequalities, we are now able to bound the distortion as follows:
\begin{align*}
\dist_{\sum}(f, \sigma)
&= \frac{\sum_{Y \in M} p(Y) \cost(Y) }{\cost(X)} \\
&\leq p(X) + \frac{\sum_{Y \neq X} p(Y)\left(\cost(X) + \left(n - \frac{2}{1+\alpha}|Y^*|\right) d(X,Y) \right) }{\cost(X)}~~\text{(Due to Ineq. (\ref{eq.1}))} \\
&= 1 + \frac{\sum_{Y \neq X} p(Y)\left(n - \frac{2}{1+\alpha}|Y^*|\right) d(X,Y) }{\cost(X)} \\
&\leq 1 + \frac{\sum_{Y \neq X} p(Y)\left(n - \frac{2}{1+\alpha}|Y^*|\right) d(X,Y) }{\frac{1}{1+\alpha} \sum_{Y \in M} |Y^*| \cdot d(X,Y)} ~~\text{(Due to Ineq. (\ref{eq.2}))} ,
\end{align*}
which gives us the desired result.
%
\end{proof}

The following theorem is our main result of this section. It states that in the worst case, the distortion of randomized dictatorship is strictly better than 3 (in fact, it is at most $3 - \frac{2}{n}$, which occurs when $\alpha = 1, |W^*| = 1$ in the theorem below). Thus, this simple randomized mechanism has better distortion than {\em any} deterministic mechanism, since no deterministic mechanism can have distortion strictly better than 3 in the worst case \cite{anshelevich2014approximating}. This is surprising for several reasons. First, randomized dictatorship only operates on the first preferences of every agent: there is no need to elicit the full preference ranking of every agent, only their top choice. Second, randomized dictatorship is strategy-proof, unlike many deterministic mechanisms. Finally, randomized dictatorship can be thought of as a randomized generalization of plurality or dictatorship. Both of these deterministic mechanisms have unbounded distortion, which means that adding some randomization significantly improves the distortion of these mechanisms. 

\begin{thm}
\label{thm:randomized_dictatorship_dec}
If a metric space is $\alpha$-decisive, then the distortion of randomized dictatorship is at most $2 + \alpha -  \frac{2|W^*|}{n}$, where $W = \arg\min_{Y \in M: |Y^*|> 0} |Y^*|$, and this bound is tight.
\end{thm}
\begin{proof}
Let $X$ be the optimal alternative. We first apply Lemma~\ref{lem:dist_bound_dec} and then use the definition of $|W^*|$:
\begin{align*}
\dist_{\sum}(f, \sigma)
&\leq 1 + \frac{(1+\alpha)\sum_{Y \in M}p(Y)(n - \frac{2}{1+\alpha}|Y^*|)d(X,Y)}{\sum_{Y \in M} |Y^*| d(X, Y)} \\
&\leq 1 + \frac{(1+\alpha)\sum_{Y \in M}p(Y)(n - \frac{2}{1+\alpha}|W^*|)d(X,Y)}{\sum_{Y \in M} |Y^*| d(X, Y)} \\
&= 1 + \frac{(1+\alpha)\sum_{Y \in M} \frac{|Y^*|}{n} \left(n - \frac{2}{1+\alpha}|W^*|\right) d(X, Y)}{\sum_{Y \in M} |Y^*| d(X, Y)} \\
&= 1 + \frac{(1+\alpha)\left(1 - \frac{2}{1+\alpha}\frac{|W^*|}{n}\right)\sum_{Y \in M} |Y^*|d(X,Y)}{\sum_{Y \in M} |Y^*| d(X, Y)} \\
&= 2 + \alpha - \frac{2|W^*|}{n}.
\end{align*}

We will now show that this bound is tight, using a generalized example of Figure~\ref{fig:randomizedDictatorshipLowerbound}. To do this, we must show there exists a preference profile induced by an $\alpha$-decisive metric space where the distortion is at least $2 + \alpha - \frac{2|W^*|}{n}$. We consider a preference profile in which there are two alternatives $W,X$ such that $|W^*| \leq |X^*|$. We will now show there exists an $\alpha$-decisive metric space that induces this profile that achieves the aforementioned distortion.

All agents $i$ who prefer $X$ have $d(i, X) = 0$ and $d(i, W) = 1$. The remaining agents have $d(i, W) = \frac{\alpha}{1+\alpha}, d(i, X) = \frac{1}{1+\alpha}$. Clearly, all of the agents are $\alpha$-decisive. We observe that $\cost(X) = \frac{|W^*|}{1+\alpha}$ and $\cost(W) = \frac{\alpha|W*|}{1+\alpha} + |X^*|$. Thus, the distortion of randomized dictatorship is
\begin{align*}
\frac{p(X)\cost(X) + p(W) \cost(W)}{\cost(X)}
&= \frac{|X^*|}{n} + \frac{|W^*|}{n}\left[\frac{\frac{\alpha}{1+\alpha}|W^*| + |X^*|}{\frac{1}{1+\alpha}|W^*|}\right] \\
&= \frac{(1+\alpha)|X^*| + \alpha|W^*| + |X^*|}{n} \\
&=  \frac{(2+\alpha)(n - |W^*|) + \alpha |W^*|}{n} \\
&= 2 + \alpha - \frac{2|W^*|}{n}.
\end{align*}
\end{proof}

\begin{figure}
\centering
\begin{tikzpicture}[scale=1,transform shape]
        \Vertex[x=0,y=0,LabelOut=true,L=$n-1$ agents,Lpos=270]{j}
  \Vertex[x=0,y=0,]{X}

  \Vertex[x=8,y=0]{W}
    \Vertex[x=4,y=0,LabelOut=true,L=One agent,Lpos=270]{i}

  \tikzstyle{LabelStyle}=[fill=white,sloped]

  \Edge[label=$(\frac{1}{2}+\epsilon)$](X)(i)
  \Edge[label=$(\frac{1}{2}-\epsilon)$](i)(W)
 \tikzstyle{EdgeStyle}=[bend left]

  \Edge[label=$1$](X)(W)
  \tikzstyle{EdgeStyle}=[bend right]
\end{tikzpicture}
\caption{Consider the case where $\alpha = 1$ and $|W^*| = 1$. There are $n-1$ agents located at $X$ who prefer $X$ and one agent between $X$ and $W$ who prefers $W$. As $\epsilon \rightarrow 0$, the worst-case distortion of randomized dictatorship approaches $3 - \frac{2}{n}$.}
\label{fig:randomizedDictatorshipLowerbound}
\end{figure}
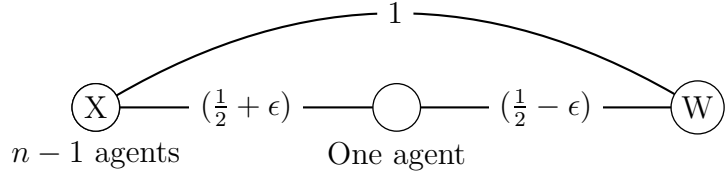

While randomized dictatorship performs well, it still does not achieve the lower bound on distortion of $1 + \alpha$ for randomized mechanisms. In general, we do not know of randomized mechanisms that can achieve this bound. However, we will now define an optimal mechanism for $\alpha$-decisive metric spaces when there are $m=2$ alternatives. This mechanism is a generalization of proportional to squares that is parameterized by $\alpha$. For $\alpha = 1$, the mechanism is in fact ordinary proportional to squares. This mechanism addresses the worst cases of randomized dictatorship by placing more probability on alternatives that receive vast majorities of the votes, if they exist.

\paragraph{$\alpha$-Generalized Proportional to Squares.} We will provide a generalization of the proportional to squares mechanism for $m=2$ that is also a function of $\alpha$. An alternative $Y$ is selected with probability $$p(Y) = \frac{(1+\alpha)|Y^*|^2 - (1-\alpha)|X^*||Y^*|}{(1+\alpha)(|X^*|^2+|Y^*|^2) - 2(1-\alpha)|X^*||Y^*|},$$
 where $X$ is the second alternative.

\begin{thm}
\label{thm:gen_proportional_squares_dec}
If a metric space is $\alpha$-decisive and $m=2$, then the distortion of $\alpha$-generalized proportional to squares is $1+\alpha$, and this is tight.
\end{thm}
\begin{proof}
Suppose $X$ is optimal, and $Y$ is the second alternative. By Lemma~\ref{lem:dist_bound_dec}, we have that the distortion is at most
$$1 + \frac{(1+\alpha)p(Y)(n - \frac{2}{1+\alpha}|Y^*|)d(X,Y)}{ |Y^*| d(X, Y)}.$$
Then, in order to bound the distortion, it suffices to simply use the fact that $n = |X^*| + |Y^*|$ and plug in $p(Y)$. We obtain a distortion of at most
\begin{align*}
&\quad  1 + \frac{(1+\alpha)p(Y)(|X^*| + |Y^*| - \frac{2}{1+\alpha}|Y^*|)d(X,Y)}{ |Y^*| d(X, Y)} \\
&= 1 + \frac{(1+\alpha)(|X^*| - \frac{1-\alpha}{1+\alpha}|Y^*|)\left((1+\alpha)|Y^*| - (1-\alpha)|X^*|\right) }{(1+\alpha)(|X^*|^2+|Y^*|^2) - 2(1-\alpha)|X^*||Y^*|} \\
&= 1 + \frac{2(1+ \alpha^2 )|X^*||Y^*| - (1+\alpha)(1-\alpha)(|X^*|^2 + |Y^*|^2)}{(1+\alpha)(|X^*|^2+|Y^*|^2) - 2(1-\alpha)|X^*||Y^*|} \\
&= \frac{(1+\alpha)\left(\alpha |Y^*|^2 + \alpha |X^*|^2 + 2\alpha |X^*||Y^*|\right)}{(1+\alpha)(|X^*|^2+|Y^*|^2) - 2(1-\alpha)|X^*||Y^*|}.
\end{align*}
In order to complete our proof, we must show that the numerator is at most a factor of $1+\alpha$ larger than the denominator. We claim that $$\alpha |Y^*|^2 + \alpha |X^*|^2 + 2\alpha |X^*||Y^*| \leq (1+\alpha)(|X^*|^2+|Y^*|^2) - 2(1-\alpha)|X^*||Y^*|.$$ This follows from the fact that $|X^*|^2 + |Y^*|^2 - 2|X^*||Y^*| = (|X^*| - |Y^*|)^2 \geq 0$. Thus, the distortion is at most $1+\alpha$, as desired.
\end{proof}

\subsection{1-Euclidean Preferences}
We now consider a well-known and well-studied special case of 1-Euclidean preferences \cite{elkind2014recognizing,procaccia2013approximate} in which all agents and alternatives are on the real number line and the metric is defined to be the Euclidean distance. First, we observe that in this setting, a Condorcet winner always exists, so the distortion is at most 3, and this is tight for deterministic mechanisms. This is true due to the results in \cite{anshelevich2014approximating}, which state that when an alternative is chosen which pairwise defeats the optimal alternative, then the distortion is at most 3. In designing an optimal randomized mechanism, we heavily use properties of this metric space from \cite{elkind2014recognizing}. Namely, using only the preference profile, we can determine the ordering of the agents on the line (which is unique up to reversal and permutations of identical voters) and the unique ordering of the alternatives that are between the top preference of the first agent and the top preference of the last agent. While this information is not enough to find the optimal alternative, using this information we will be able to significantly reduce the set of alternatives that can be optimal. Then we will use $\alpha$-generalized proportional to squares on this restricted set of alternatives to achieve a better distortion bound. Our full mechanism is shown below as Algorithm~\ref{alg:1euclidean}.

\begin{algorithm}[h]
\caption{Optimal randomized mechanism for the $\alpha$-decisive, 1-Euclidean space}
\begin{algorithmic}
\REQUIRE A preference profile $\sigma$
\ENSURE A probability distribution $p$ over the alternatives
\STATE $>_N \gets$ ordering of the agents \cite{elkind2014recognizing}
\STATE $>_M \gets$ ordering of the alternatives \cite{elkind2014recognizing}
\STATE $i' \gets $ median agent of $>_N$
\STATE $X \gets$ top preference of $i'$
\STATE $Y \gets$ alternative directly left of $X$ in $>_M$
\STATE $Z \gets$ alternative directly right of $X$ in $>_M$
\IF{$|YX| < |ZX|$}
\vskip 2pt
\STATE $p(Z) \gets \ddfrac{(1+\alpha)|ZX|^2 - (1-\alpha)|XZ||ZX|}{(1+\alpha)(|XZ|^2+|ZX|^2) - 2(1-\alpha)|XZ||ZX|}$
\vskip 2pt
\STATE $p(X) \gets \ddfrac{(1+\alpha)|XZ|^2 - (1-\alpha)|XZ||ZX|}{(1+\alpha)(|XZ|^2+|ZX|^2) - 2(1-\alpha)|XZ||ZX|}$
\vskip 2pt
\ELSIF{$|YX| > |ZX|$}
\vskip 2pt
\STATE $p(Y) \gets \ddfrac{(1+\alpha)|YX|^2 - (1-\alpha)|XY||YX|}{(1+\alpha)(|XY|^2+|YX|^2) - 2(1-\alpha)|XY||YX|}$
\vskip 2pt
\STATE $p(X) \gets \ddfrac{(1+\alpha)|XY|^2 - (1-\alpha)|XY||YX|}{(1+\alpha)(|XY|^2+|YX|^2) - 2(1-\alpha)|XY||YX|}$
\vskip 4pt
\ELSE
\STATE $p(X) \gets 1$
\ENDIF
\end{algorithmic}
\label{alg:1euclidean}
\end{algorithm}

We will now show that this mechanism has worst-case distortion at most $1+\alpha$ through a series of steps in which we reduce the set of possible optimal alternatives from $m$ to 2. In our first lemma, we show that the optimal alternative must be one of the two alternatives on either side of the median agent from our agent ordering. One of these alternatives must be the top preference of the median agent. However, since we do not know if the median agent's top preference is to the left or right of it, we must consider three alternatives: her top preference and the two alternatives on either side of the top preference. This reduces our set of optimal alternatives from $m$ to 3.

\begin{lem}\label{lem:top3}
In the 1-Euclidean setting, consider the median agent $i'$. Let this agent's top preference be $X$. Call the alternatives directly to the left and right of this alternative $Y$ and $Z$, respectively. Then $X, Y $ or $Z$ must be optimal.
\end{lem}
\begin{proof}
Suppose that $X$ is to the left of the median voter. Then $X$ has $x \geq \frac{n}{2}$ agents to the right of it who prefer $X$ over $Y$. For these agents $i$, $d(i,Y) = d(i,X) + d(X,Y)$, while the remaining $n-x$ agents $i$ have $d(i,X) \leq d(i,Y) + d(X, Y)$. Thus, $\sum_{i\in N} d(i,X) \leq \sum_{i \in N} d(i,Y) - x\cdot d(X,Y) + (n-x) \cdot d(X,Y) \leq \sum_{i \in N} d(i,Y) $, which implies that the quality of $X$ is always at least as good as $Y$. This same argument can be used for any alternative to the left of $X$.

We observe that if $X$ is left of the median voter, $Z$ must be to the right of the median voter because if not, then the median voter would prefer $Z$ over $X$. Since at least $\frac{n}{2}$ agents to the left of $Z$ prefer it over any alternative to the right of it, we can use the same argument to show that $Z$ is better than all of these alternatives. Thus, $X$ or $Z$ must be optimal.

Finally, if $X$ to the right of the median voter, we can show that $X$ or $Y$ must be the optimal alternative. However, since it is not possible to determine if $X$ is to the left or right of the median voter, then we know that one of $X$, $Y$, or $Z$ must be optimal.
\end{proof}

Next, we show that we can further reduce the set of possible optimal alternatives from 3 to 2.

\begin{lem}
\label{lem:1euclidean_3to2}
If $|YX| \leq |ZX|$, then $Y$ cannot be better than $X$, and if $|ZX| \leq |YX|$, $Z$ cannot be better than $X$.
\end{lem}
\begin{proof}
Suppose, without loss of generality, that $|YX| \leq |ZX|$. Then, since all agents in $ZX$ must be to the right of $X$, we have that
\begin{align*}
\sum_{i \in N} d(i,X) &= \sum_{i \in ZX} (d(i,Y) - d(X,Y)) + \sum_{i \notin ZX} d(i,X) \\
&\leq \sum_{i \in ZX} (d(i,Y) - d(X,Y)) + \sum_{i \in X^*} d(i,Y) + \sum_{i \in YX} (d(i,Y) + d(X,Y)) \\
&= \sum_{i \in N} d(i,Y) - |ZX| d(X, Y) + |YX| d(X,Y) \\
&\leq \sum_{i \in N} d(i,Y).
\end{align*}
Note that $ZX$ and $YX$ are disjoint, since agents in $ZX$ must be to the right of $X$ and agents in $YX$ must be to the left of $X$. Because of this, the third transition above is an equality. Thus, we have shown that $Y$ cannot be better than $X$.
\end{proof}

Finally, we can use the $\alpha$-generalized proportional to squares mechanism on the restricted set of alternatives $X$ and one of $Y, Z$ to achieve a distortion of $1+\alpha$, which is tight since our lower bound example from Theorem \ref{thm:randomizedLowerBound} occurs in the 1-Euclidean setting. In the event that $|YX| = |ZX|$, then we can select $X$ with probability 1, since neither $Y$ nor $Z$ can be better than $X$.

\begin{thm}
In the 1-Euclidean setting, Algorithm \ref{alg:1euclidean} has distortion at most $1 + \alpha$, and thus has the best possible worst-case distortion.
\end{thm}
\begin{proof}
Let $X, Y, Z$ be as defined in the algorithm. If $|YX|=|ZX|$, then by Lemma~\ref{lem:1euclidean_3to2} it must be that $X$ has better social cost than $Y$ or $Z$, and by Lemma \ref{lem:top3}, this means that $X$ must be the optimum outcome. Therefore, our algorithm selects $X$ with probability 1, and achieves distortion of 1.

Now suppose that $|YX| > |ZX|$, without loss of generality. By Lemma~\ref{lem:1euclidean_3to2} this means that $X$ has better social cost than $Z$ and that one of $X$ or $Y$ must be the optimum alternative.

Assume that $X$ is optimal instead of $Y$ (the proof of the other case is identical). Since we are in the 1-Euclidean setting, we know that every agent in $YX \setminus  Y^*$ is to the left of $Y$. Therefore, for all $i \in YX \setminus  Y^*$, $d(i, X) = d(i, Y) + d(X, Y)$. Using this fact, as well ad the definition of $\alpha$-decisiveness and the triangle inequality which states that $d(i,Y)\leq d(i,X)+d(X,Y)$, we can derive an improved upper bound on the social cost of $Y$:
\begin{align*}
&\quad \cost(Y) = \sum_{\i \in N} d(i,Y) \\
&= \sum_{i\in Y^*}d(i,Y) + \sum_{i\in YX \setminus Y^*} d(i,Y) + \sum_{i\in XY} d(i,Y) \\
&\leq \alpha\sum_{i \in Y^*} d(i,X) + \sum_{i \in YX \setminus  Y^*} (d(i,X) - d(X,Y)) + \sum_{i \in XY} (d(i,X) + d(X,Y))
\end{align*}

We continue to derive a better bound on the social cost of $Y$ from the above; the second inequality below is due to Lemma \ref{lem:dec_agent_lb}.

\begin{align*}
&\quad \cost(Y) \leq \\
&\leq \sum_{i \in N} d(i,X) - (1-\alpha) \sum_{i \in Y^*} d(i,X) + \left(|XY| - |YX \setminus Y^*| \right) d(X,Y) \\
&= \cost(X) - (1-\alpha) \sum_{i \in Y^*} d(i,X) +\left(|XY| - |YX \setminus Y^*| \right) d(X,Y) \\
&\leq \cost(X) - \frac{1-\alpha}{1+\alpha} \sum_{i \in Y^*} d(X,Y) + \left(|XY| - |YX \setminus Y^*| \right) d(X,Y) \\
&= \cost(X) + \left(|XY| - |YX \setminus Y^*| - \frac{1-\alpha}{1+\alpha}|Y^*| \right) d(X,Y) \\
&= \cost(X) + \left(n - \frac{2}{1+\alpha}|YX| - \frac{2\alpha}{1+\alpha}|YX \setminus Y^*| \right) d(X,Y).
\end{align*}
We can also derive an improved lower bound for the social cost of $X$, the last inequality below is again due to Lemma \ref{lem:dec_agent_lb}.
\begin{align*}
 \cost(X) &= \sum_{i \in Y^*} d(i,X) + \sum_{i \in YX \setminus  Y^*} d(i,X) + \sum_{i \in XY} d(i,X)  \\
&\geq  \sum_{i \in Y^*} d(i,X) + \sum_{i \in YX \setminus  Y^*} d(i,X) \\
&= \sum_{i \in Y^*} d(i,X) + \sum_{i \in YX \setminus  Y^*} (d(i,Y) + d(X,Y)) \\
&\geq \sum_{i \in Y^*} d(i,X) + \sum_{i \in YX \setminus  Y^*} d(X,Y) \\
&\geq \frac{1}{1+\alpha} |Y^*|d(X,Y) + |YX \setminus Y^*| d(X,Y) \\
&= \frac{1}{1+\alpha}\left(|YX| + \alpha|YX \setminus Y^*| \right) d(X,Y)
\end{align*}
We can now bound the distortion using these two inequalities. We will demonstrate that the distortion is maximized when there are no agents to the left of $Y$, i.e., $|YX \setminus Y^*|= 0 \Rightarrow |Y^*| = |YX|$. As we have seen, the distortion is also maximized when $\forall i \in |XY|, d(i,X) = 0$. Thus, we will have effectively reduced the problem to the case where $m = 2$.
\begin{align*}
&\quad \frac{p(Y)SC(Y) + p(X)SC(X)}{SC(X)}  \\
&\leq \frac{p(Y)\left(\cost(X) + \left(n - \frac{2}{1+\alpha}|YX| - \frac{2\alpha}{1+\alpha}|YX \setminus Y^*| \right) d(X,Y) \right) + p(X) \cost(X)}{\cost(X)} \\
&= 1+ \frac{p(Y)\left(n - \frac{2}{1+\alpha}|YX| - \frac{2\alpha}{1+\alpha}|YX \setminus Y^*| \right) d(X,Y)}{\cost(X)} \\
&\leq  1+ \frac{p(Y)\left(n - \frac{2}{1+\alpha}|YX| - \frac{2\alpha}{1+\alpha}|YX \setminus Y^*| \right) d(X,Y)}{\frac{1}{1+\alpha}\left(|YX| + \alpha|YX \setminus Y^*| \right) d(X,Y) } \\
&= 1 + \frac{p(Y)\left((1+\alpha)n - 2|YX| - 2\alpha|YX \setminus Y^*| \right)}{|YX| + \alpha|YX \setminus Y^*|}
\end{align*}
If we hold $|YX|$ constant, we observe that the distortion is decreasing in $|YX \setminus Y^*|$. In other words, the distortion is maximized when $|YX| = |Y^*|$. If we set $|YX \setminus Y^*| = 0$ in the above distortion bound, then the remainder of this proof proceeds identically as in the proof of Theorem~\ref{thm:gen_proportional_squares_dec}.\end{proof}





\subsection{Distortion in the $(m-1)$-Simplex}\label{sec:simplex}

In this section we consider a specialized, yet natural, setting inspired by \cite{anshelevich2014approximating}, known as the \emph{$(m-1)$-simplex setting}. In this setting, we assume that the $m$ alternatives are all at distance 1 from each other and for every agent $i$, for all $Y \in M$, we have that $d(i,Y) \leq 1$. This includes the case when $m$ alternatives are the vertices of the $(m-1)$-simplex and all of the agents lie inside this simplex. Although this is a very constrained setting, it is a reasonable assumption in the case when all of the alternatives are uncorrelated, i.e., when all the alternatives are equally different from one another. For example, when choosing where to allocate money, it may be that for a die-hard fan of alternative $X$, all other alternatives are equally bad. When choosing which of three services $X$, $Y$, or $Z$ to improve, someone who only uses service $X$ will have equally large cost for both alternative $Y$ or $Z$, since they do not benefit from them. On the other hand, a person using all three services equally may be indifferent between the alternatives. When the question is "Should we improve the highway system in New York, California, or Texas?", someone who lives in New York would want their roads improved, not someone else's, while someone who often visits all three states may be more indifferent between the alternatives.

In this setting, the distortion of randomized dictatorship does not improve because the worst case occurs on a line. However, we will see that plurality and the proportional to squares mechanism are good for any number of alternatives in this setting.

\begin{thm}
\label{thm:plurality}
If the $(m-1)$-simplex setting is $\alpha$-decisive, then plurality has distortion at most $1 + 2\alpha.$
\end{thm}
\begin{proof}
Suppose $X$ is the optimal alternative, and $W$ is the alternative selected by plurality. For convenience, define $\delta = \sum_{i\in W^*} d(i,X)$.

By Lemma~\ref{lem:dec_agent_lb}, we know that $\cost(X) \geq \delta + \frac{1}{1+\alpha}\sum_{Y \neq X,W} |Y^*| d(X,Y)$, which equals $\delta + \frac{1}{1 + \alpha}\left(n - |X^*| - |W^*|\right)$ since in the simplex setting distances between all alternatives equal 1.
Furthermore, since in the simplex setting all distances $d(i,W)$ are at most one, we have that $\cost(W)\leq n - |W^*| + \sum_{i\in W^*} d(i,W) \leq n-|W^*|+\alpha\delta$; the last inequality is due to the definition of $\alpha$-decisiveness.

Putting this together, we have that the distortion is at most
$$\frac{\alpha\delta + (n-|W^*|)}{\delta + \frac{1}{1 + \alpha}\left(n - |X^*| - |W^*|\right)}.$$

This bound is decreasing with $\delta$ (since $n-|W^*| > \alpha\cdot \frac{1}{1 + \alpha}\left(n - |X^*| - |W^*|\right)$), and so is maximized for $\delta$ being as small as possible. By the triangle inequality, and the fact that $d(X,W)=1$ for simplex settings, we know that $1\leq d(i,W)+d(i,X)$ for each $i\in W^*$. Since $d(i,W)\leq \alpha\cdot d(i,X)$ by definition of decisiveness, this means that $1\leq (1+\alpha)d(i,X)$, and thus $\delta \geq    \frac{1}{1+\alpha}|W^*|$. Plugging this into the expression above, we obtain that the distortion is at most
\begin{align*}
\frac{\frac{\alpha}{1+\alpha}|W^*| + (n-|W^*|)}{\frac{1}{1+\alpha}|W^*| + \frac{1}{1 + \alpha}\left(n - |X^*| - |W^*|\right)}
&= \frac{(1+\alpha)n - |W^*|}{n-|X^*|} \\
&\leq \frac{(1+\alpha)n - |X^*|}{n-|X^*|} \\
&= 1 + \frac{\alpha \cdot n}{n - |X^*|}.
\end{align*}
Since $|X^*| \leq |W^*|$, it follows that $|X^*| \leq \frac{n}{2}$. Thus, $1 + \frac{\alpha \cdot n}{n - |X^*|} $, which is increasing in $|X^*|$, is maximized when $|X^*| = \frac{n}{2}$. We conclude that the distortion is at most $1 + 2\alpha$.
\end{proof}

Plurality, although a deterministic mechanism, does very well in this setting, because for $\alpha = 0$ the alternative with the most votes is clearly optimal. In general, as $\alpha \to 0$, the agents are forced closer to the vertices of the simplex, and plurality better approximates finding the optimal alternative. However, when $\alpha$ is not small, plurality fares poorly compared to the proportional to squares mechanism. Indeed, for $\alpha \leq \frac{1}{7}$, plurality has a better upper bound on distortion than proportional to squares, but otherwise the opposite is true. The difference becomes more obvious for high $\alpha$: for $\alpha = 1$ (i.e., for general metrics), proportional to squares has distortion at most 2, while plurality has distortion at most 3.

\begin{thm}
\label{thm:simplex_dec}
If the $(m-1)$-simplex setting is $\alpha$-decisive, the proportional to squares mechanism has distortion at most $\frac{1}{2}\left(1 + \alpha + \sqrt{2}\sqrt{\alpha^2 + 1} \right).$
\end{thm}
\begin{proof}
Suppose $X$ is the optimal alternative. For convenience, define $\delta_Y = \sum_{i\in Y^*} d(i,X)$ for any alternative $Y\neq X$. We begin by proceeding identically to the proof of Theorem \ref{thm:plurality}. Thus we obtain that, for all $Y\neq X$, it holds that
$$\frac{\cost(Y)}{\cost(X)}\leq \frac{\alpha\delta_Y + (n-|Y^*|)}{\delta_Y + \frac{1}{1 + \alpha}\left(n - |X^*| - |Y^*|\right)},$$
and furthermore that
$$\frac{\cost(Y)}{\cost(X)} \leq \frac{(1+\alpha)n - |Y^*|}{n-|X^*|}.$$
This is because until this point in the proof of Theorem \ref{thm:plurality}, we do not use anywhere that we are comparing $X$ with the outcome chosen by plurality; this comparison holds for the costs of arbitrary alternatives.

Now, let $\gamma = \left(n - |X^*|\right)  \cdot \sum_{Z \in M} |Z^*|^2$. If $p(Y)$ is the probability of alternative $Y$ being chosen by the proportional to squares mechanism, then the expected distortion of this mechanism is equal to
\begin{align*}
&\quad \frac{\sum_{Y \in M} p(Y)\cost(Y)}{\cost(X)} \\
&= p(X) + \sum_{Y\neq X}\frac{|Y^*|^2}{\cdot \sum_{Z \in M} |Z^*|^2} \left(\frac{\cost(Y)}{\cost(X)} \right) \\
&\leq p(X) + \sum_{Y\neq X}\frac{|Y^*|^2}{\cdot \sum_{Z \in M} |Z^*|^2} \left(\frac{(1+\alpha)n - |Y^*|}{n-|X^*|} \right) \\
&= p(X) + \frac{1+\alpha}{\gamma}\sum_{Y\neq X}\left(\frac{\alpha}{1+\alpha}|Y^*|^3 +  |Y^*|^2\sum_{Z\neq Y} |Z^*| \right) \\
&= \frac{1+\alpha}{\gamma} \left(\frac{1}{1+\alpha}|X^*|^2(n-|X^*|) + \sum_{Y\neq X}\left(\frac{\alpha}{1+\alpha}|Y^*|^3 +  |Y^*|^2\sum_{Z\neq Y} |Z^*| \right) \right) \\
&= \frac{1+\alpha}{\gamma} \left(\frac{1}{1+\alpha}|X^*|^2 \sum_{Y \neq X} |Y^*| + |X^*|\sum_{Y\neq X}|Y^*|^2 + \sum_{Y\neq X}\left(\frac{\alpha}{1+\alpha}|Y^*|^3 +  |Y^*|^2\sum_{Z\neq Y,X} |Z^*| \right) \right).
\end{align*}


Let $\beta = 1 + \frac{\sqrt{2}\sqrt{\alpha^2 + 1} }{\alpha+1} $. If we want to obtain the desired bound of $ \frac{1}{2}\left(1 + \alpha + \sqrt{2}\sqrt{\alpha^2 + 1} \right)$, we must show that
\begin{align*}
&\quad \frac{1}{1+\alpha}|X^*|^2 \sum_{Y \neq X} |Y^*| + |X^*|\sum_{Y\neq X}|Y^*|^2 + \sum_{Y\neq X}\left(\frac{\alpha}{1+\alpha}|Y^*|^3 +  |Y^*|^2\sum_{Z\neq Y,X} |Z^*| \right)\\
&\leq \frac{1}{2}\left(1 + \alpha + \sqrt{2}\sqrt{\alpha^2 + 1} \right)\frac{\gamma}{1+\alpha} \\
&= \frac{1}{2}\beta \left(n - |X^*|\right)  \sum_{Z \in M} |Z^*|^2 \\
&= \frac{1}{2}\beta \left(|X^*|^2 \sum_{Y\neq X} |Y^*| + \sum_{Y\neq X} \left(|Y^*|^3 + |Y^*|^2\sum_{Z\neq Y,X} |Y^*| \right) \right)
\end{align*}

We can further simplify this inequality by canceling terms on both sides. We note that $\frac{1}{2}\beta \geq 1 \geq \frac{1}{1+\alpha} \geq \frac{\alpha}{1+\alpha}$.  First, we consider an alternative $Y \neq X$. On the LHS, $|Y^*|^3$ terms have a coefficient of $\frac{\alpha}{1+\alpha}$, while on the RHS, they have a coefficient of $\frac{1}{2}\beta$. We subtract $\frac{\alpha}{1+\alpha}|Y^*|^3$ from both sides. Similarly, for $|Y^*|^2 \sum_{Z\neq Y,X} |Z^*|$ terms, we have a coefficient of 1 on the LHS and a coefficient of $\frac{1}{2}\beta$ on the RHS. We subtract $|Y^*|^2 \sum_{Z\neq Y,X} |Z^*|$ from both sides. We repeat this process for all $Y \neq X$.

Next we consider terms that contain $|X^*|$. We observe that neither side has $|X^*|^3$ terms. The term $|X^*|^2 \sum_{Y\neq X}|Y^*|$ has a coefficient of $\frac{1}{1+\alpha}$ on the LHS, while it has a coefficient of $\frac{1}{2}\beta$ on the RHS. Thus, we subtract $\frac{1}{1+\alpha}|X^*|^2 \sum_{Y\neq X}|Y^*|$ from both sides. Finally, we consider the term $|X^*|\sum_{Y\neq X} |Y^*|^2$. The LHS has this term with a coefficient of 1, while the RHS does not have this term at all. We note that this is the only term remaining on the LHS after cancellation.

After all of this canceling, this leaves us with needing to prove that
\begin{equation*}
|X^*|\sum_{Y \neq X} |Y^*|^2 \leq  \frac{1}{2}\sum_{Y\neq X}\left(\left(\beta - \frac{2\alpha}{1+\alpha} \right)|Y^*|^3 + \left(\beta -  \frac{2}{1+\alpha}\right)|Y^*||X^*|^2  + \left(\beta - 2 \right) |Y^*|^2 \sum_{Z\neq X,Y} |Z^*|  \right).
\end{equation*}

In fact, we will prove that
\begin{equation*}
|X^*|\sum_{Y \neq X} |Y^*|^2 \leq  \frac{1}{2}\sum_{Y\neq X}\left(\left(\beta - \frac{2\alpha}{1+\alpha} \right)|Y^*|^3 + \left(\beta -  \frac{2}{1+\alpha}\right)|Y^*||X^*|^2   \right),
\end{equation*}
which will imply the above inequality, since $\sum_{Y\neq X}\left(\beta - 2 \right) |Y^*|^2 \sum_{Z\neq X,Y} |Z^*|$ is always non-negative.

We observe that
\begin{align*}
&\quad \sum_{Y\neq X}\left(\left(\beta - \frac{2\alpha}{1+\alpha} \right)|Y^*|^3 + \left(\beta -  \frac{2}{1+\alpha}\right)|Y^*||X^*|^2  - 2|Y^*|^2|X^*| \right) \\
&= \sum_{Y\neq X}|Y^*|\left(\left(\beta - \frac{2\alpha}{1+\alpha} \right)|Y^*|^2 + \left(\beta -  \frac{2}{1+\alpha}\right)|X^*|^2  - 2|Y^*||X^*| \right),
\end{align*}
which we claim is non-negative. Proving that this quantity is non-negative completes our proof. First, we claim that it follows from simple algebra that the product of $\left(\beta -  \frac{2}{1+\alpha}\right)$ and $\left(\beta - \frac{2\alpha}{1+\alpha} \right)$ is 1. Thus, for any $Y \neq X$, we can show that
\begin{equation*}
\left(\beta - \frac{2\alpha}{1+\alpha} \right)|Y^*|^2 + \left(\beta -  \frac{2}{1+\alpha}\right)|X^*|^2 - 2|Y^*||X^*| = \left( |Y^*|\sqrt{\beta - \frac{2\alpha}{1+\alpha}} - |X^*|\sqrt{\beta -  \frac{2}{1+\alpha}} \right)^2 \geq 0,
\end{equation*}
which implies that the summation over these terms is non-negative as well.
\end{proof}

Unlike all of the previous distortion bounds we have provided, this is the first that is not linearly increasing in $\alpha$. It increases slower than the distortion of plurality, which is $1 + 2\alpha$. For smaller values of $\alpha$, such as $\alpha = 0$, which is where plurality has the largest advantage over proportional to squares, the distortion of proportional to squares is still at most $\frac{1 + \sqrt{2}}{2} \approx 1.2071$, which is reasonably small. For $1 \geq \alpha \geq .5$, the values of $1 + \alpha$ and $\frac{1}{2}\left(1 + \alpha + \sqrt{2}\sqrt{\alpha^2 + 1} \right)$ are relatively close. Since we have $1+\alpha$ as a lower bound for all randomized mechanisms, this implies that proportional to squares is nearly optimal for sufficiently large values of $\alpha$. We suspect that the optimal mechanism in the $(m-1)$-simplex setting is in fact a modified version of $\alpha$-generalized proportional to squares that works for arbitrary $m$, and we think it should have a distortion upper bound of $1 + \alpha$.

\section{Median Agent Cost}

\subsection{General Metric Spaces}


In this section, we will examine the median objective function. In \cite{anshelevich2014approximating}, it was shown than no deterministic mechanism can achieve a worst-case distortion of better than 5, and that the Copeland mechanism achieves this bound. We begin this section by showing that randomized mechanisms have a general worst-case distortion lower bound of 3 rather than 5 like deterministic mechanisms.

\begin{thm}
\label{thm:median_random_lb}
For $m \geq 2$, the worst-case median distortion is at least $3$ for all randomized mechanisms.
\end{thm}
\begin{proof}
We must show there exists a preference profile such that for all randomized mechanisms, there always exists a metric space that induces the preference profile and where the distortion is at least $3$. We will consider a preference profile with two alternatives $W,X$ and $n$ agents. In this profile, there are $n-1$ agents that prefer $W$ over $X$, while the remaining agent prefers $X$ over $W$. We claim that no randomized mechanism can achieve distortion $< 3$ for all metric spaces that induce this profile.

First, we claim that there exist metric spaces where the distortion is unbounded if $X$ is picked with any positive probability. For example, suppose for all agents that prefer $W$ over $X$, $d(i, W) = 0, d(i,X) = 1$. The agent that prefers $X$ over $W$ has $d(i,W) = 1, d(i,X) = 0$. Thus, $\med(W) = 0$ and $\med(X) = 1$. Thus, we conclude that any randomized mechanism with $p(X) > 0$ for the given preference profile has unbounded worst-case distortion.

In order to complete our proof, we only need to consider randomized mechanisms that select $W$ with probability 1, given the aforementioned preference profile. We will show there exists a metric space in which the distortion is at least 3 for these mechanisms. Consider the following metric space: there are $\frac{n}{2}$ agents with $d(i, W) = \frac{1}{2}-\epsilon$ and $d(i, X) = \frac{1}{2}+\epsilon$. One agent has $d(i, W) = \frac{3}{2}, d(i, X) = \frac{1}{2}$. The remaining agents who prefer $W$ have $d(i, W) > 2, d(i, X) > 2$. Then $\med(X) = \frac{1}{2}+\epsilon$ and $\med(W) = \frac{3}{2}$. Since $\med(W)$ approaches $3 \cdot \med(X)$ as $\epsilon\rightarrow 0$, and $p(W) = 1$, the distortion approaches 3.
\end{proof}

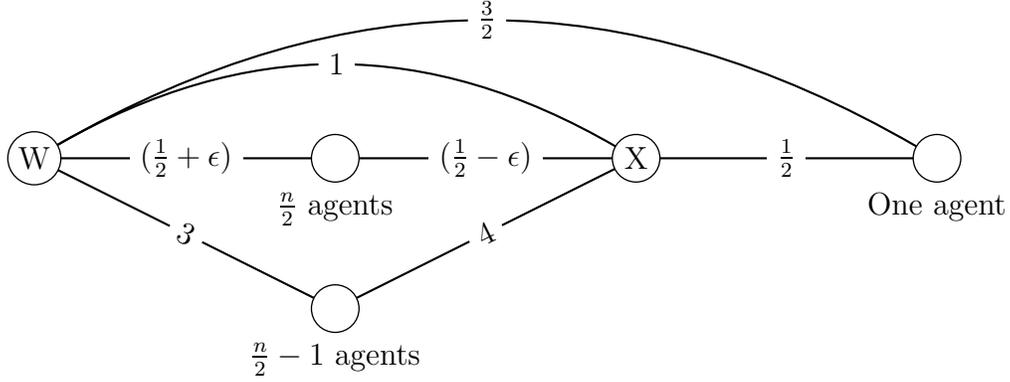
\begin{figure}
\centering
\begin{tikzpicture}[scale=1,transform shape]
  \Vertex[x=0,y=0,]{W}
  \Vertex[x=8,y=0]{X}
  \Vertex[x=4,y=0,LabelOut=true,L=$\frac{n}{2}$ agents,Lpos=270]{i}
   \Vertex[x=12,y=0,LabelOut=true,L=One agent,Lpos=270]{j}

      \Vertex[x=4,y=-2,LabelOut=true,L=$\frac{n}{2} - 1$ agents,Lpos=270]{k}
  \tikzstyle{LabelStyle}=[fill=white,sloped]
  \Edge[label=$(\frac{1}{2}+\epsilon)$](W)(i)
  \Edge[label=$(\frac{1}{2}-\epsilon)$](i)(X)
  \Edge[label=$\frac{1}{2}$](j)(X)
  \Edge[label=$4$](k)(X)
  \Edge[label=$3$](k)(W)
  \tikzstyle{EdgeStyle}=[bend left]
  \Edge[label=$1$](W)(X)

  \tikzstyle{EdgeStyle}=[bend right]
    \Edge[label=$\frac{3}{2}$](j)(W)
\end{tikzpicture}
\caption{There are $\frac{n}{2}$ agents between $W$ and $X$ that prefer $W$, one agent to the right of $X$ that prefers $X$, and $\frac{n}{2} - 1$ agents who are far from both alternatives but prefer $W$. If a social choice function selects $W$ with probability 1, then the worst-case median distortion is at least 3.}
\label{fig:medianLowerBound}
\end{figure}

From this example, we are able to conclude that both randomized dictatorship and proportional to squares have unbounded distortion, even for $m=2$.

We now present the main result of this section: designing a randomized mechanism which will always achieve a distortion of at most 4 for the median objective. We claim that to design a randomized mechanism for the median objective, it makes sense to consider the uncovered set, which is the set of alternatives $X$ that pairwise defeats every other alternative $Y$ either directly (i.e. $X$ pairwise defeats $Y$) or indirectly through another alternative $Z$ (i.e. $X$ does not pairwise defeat $Y$, but $X$ pairwise defeats $Z$, which in turn pairwise defeats $Y$).
From \cite{anshelevich2014approximating}, we have the following two lemmas concerning the quality of alternatives in the uncovered set.

\begin{lem}[\cite{anshelevich2014approximating}]
\label{lem:median_condorcet}
If an alternative $W$ pairwise defeats (or pairwise ties) the alternative $X$, then $\med(W) \leq 3 \cdot \med(X)$ for all metric preferences.
\end{lem}

\begin{lem}[\cite{anshelevich2014approximating}]
\label{lem:median_uncovered}
If an alternative $W$ is in the uncovered set, then $\med(W) \leq 5\cdot \med(X)$ for all metric preferences, where $X$ is any alternative.
\end{lem}

These two lemmas suggest that if we want to achieve distortion better than 5, we should not deterministically pick a single alternative from the uncovered set because we do not know of a way to ensure we do not pick an alternative that is a factor of 5 away. Indeed, this is what can happen with Copeland. Instead, we want to mix over the entire uncovered set and ensure that some alternatives that pairwise defeat the optimal alternative (i.e., alternatives only a factor of 3 away) are chosen with high probability to decrease the distortion. However, since we do not know the optimal alternative, we must have this property hold for every alternative. This is made precise in the following theorem. Let $G = (M,E)$ be the majority graph, i.e., a graph in which the alternatives are vertices and the edges denote pairwise victories: an edge $(Y,Z) \in E$ if $Y$ is preferred to $Z$ by a strict majority of the voters. Let $S$ be the uncovered set, and $p$ be some probability distribution over $S$. Finally, define $\pi(Y)$ for any alternative $Y$ to be the total probability distribution ``covered" by $Y$, i.e., $\pi(Y)=\sum_{(Y,Z)\in E} p(Z)$. Then, we have the following statement.

\begin{thm}\label{thm:covers}
If a mechanism selects alternatives only from the uncovered set $S$ with probability distribution $p$, and if for all alternatives $X$ we have that $\pi(X)\leq \frac{1}{2}$, then the expected median distortion of this mechanism is at most 4.
\end{thm}
\begin{proof}
Let $G = (M,E)$ be the majority graph in which ties are broken arbitrarily, and let $X$ be the optimal alternative. By Lemmas~\ref{lem:median_condorcet} and \ref{lem:median_uncovered}, we know the expected distortion is at most $$1\cdot p(X) + \sum_{Z\in S: (Z,X)\in E}3 \cdot  p(Z) + \sum_{Z\in S: (X,Z) \in E}5 \cdot p(Z).$$

Since $\pi(X)\leq \frac{1}{2}$, this means that $\sum_{Z\in S: (X,Z) \in E}  p(Z) \leq \frac{1}{2}$. Since the distortion can be at worst 5 with probability $\frac{1}{2}$ and otherwise has distortion at most 3, we conclude that the distortion of this mechanism is at most 4.
\end{proof}

Thus, we want a mechanism that manages to ensure that for every alternative $X$, the alternatives that can be more than a factor of 3 away from $X$ (i.e., the ones it pairwise defeats) are selected with probability at most $\frac{1}{2}$. The mechanism we describe, Uncovered Set Min-Cover, uses a linear program to accomplish this. We define the subset of the edges on the uncovered set as $E(S) = \{E = (Y,Z): Y\in S, Z\in S \}$. We also give the LP (and its dual which is not used by the algorithm, but is necessary for our proofs), which is used as a subroutine by our mechanism.

\begin{align*}
\text{(Linear Program)}&
&\text{(Dual)}& \\
\text{minimize} \quad p_{\max}&
&\text{maximize} \quad b_{\min}& \\
\text{subject to} \quad  p_Y &\geq 0, \quad Y \in S
&\text{subject to} \quad b_Y &\geq 0, \quad Y \in S \\
\sum_{(Y, Z) \in E(S)} p_{Z} &\leq p_{\max}, \quad Y\in S
&\sum_{(Z, Y) \in E(S)} b_{Z} &\geq b_{\min} , \quad Y \in S \\
\sum_{Z \in S} p_Z &= 1. \quad\
&\sum_{Z \in S} b_Z &= 1. \quad
\end{align*}

\begin{algorithm}[h!]
\label{alg:uncoverset}
\caption{Uncovered Set Min-Cover}
\begin{algorithmic}
\REQUIRE A preference profile $\sigma$
\ENSURE A probability distribution $\vec{p}$ over the alternatives of the uncovered set
\STATE $G = (M,E) \gets$ majority graph of $\sigma$
\STATE $S \gets$ uncovered set of $G$
\STATE $\vec{p} \gets$ solution to LP (see above)
\end{algorithmic}
\end{algorithm}

Now we must show that this mechanism actually has low distortion, i.e., the following theorem.

\begin{thm}
\label{thm:uncoveredset_distortion}
The expected median distortion of Uncovered Set Min-Cover is at most $ 4$.
\end{thm}
%

This theorem is immediate from Theorem \ref{thm:covers} if we can show that for the distribution formed by Uncovered Set Min-Cover, we have that $\pi(X)\leq \frac{1}{2}$ for all $X$. We prove this fact using the following two lemmas.


\begin{lem}
\label{lem:prob_mass}
Let $G = (M,E)$ be the majority graph in which ties are broken arbitrarily. For the dual of LP,
it must be that $b_{\min} \leq \frac{1}{2}$.
\end{lem}
\begin{proof}
Suppose, by way of contradiction, that for all $Y \in S$, $\sum_{(Z, Y) \in E(S)} b_{Z} > \frac{1}{2} $, which implies that $\sum_{(Y, Z) \in E(S)} b_{Z} < \frac{1}{2}$. Then we can derive that
\begin{align*}
\frac{1}{2}
&< \sum_{Y\in S} b_Y \sum_{(Z, Y) \in E(S)} b_{Z} \\
&= \sum_{Y\in S}\sum_{(Z, Y) \in E(S)} b_{Y} b_{Z} \\
&= \sum_{(Z, Y) \in E(S)} b_{Y} b_{Z} \\
&= \sum_{Y\in S} b_Y \sum_{(Y, Z) \in E(S)} b_{Z} \\
&< \frac{1}{2},
\end{align*}
which is a contradiction.
\end{proof}

Due to LP-Duality, the above lemma immediately implies that $p_{\max}\leq \frac{1}{2}$, and thus $\pi(X)\leq\frac{1}{2}$ for all $X\in S$. This does not complete the proof of Theorem \ref{thm:uncoveredset_distortion}, however, since it is possible that the optimal alternative $X$ is outside of the uncovered set $S$. To finish the proof of the theorem, we also need the following lemma.

\begin{lem}\label{lem:outsideOfUncovered}
Suppose we have a probability distribution $p$ over alternatives in the uncovered set $S$, and for all $Y\in S$, we have that $\pi(Y)=\sum_{(Y,Z)\in E} p(Z)\leq \frac{1}{2}$. Then, this also must hold for alternatives outside of $S$, i.e., for all $X\not\in S$, we also have that $\pi(X)\leq \frac{1}{2}$.
\end{lem}
\begin{proof}
Consider an alternative $X\not\in S$, and let $U$ be the set of alternatives ``covered" by $X$ in $S$, i.e., that $X$ pairwise defeats. Suppose to the contrary that $\pi(X)>\frac{1}{2}$, i.e., that $\sum_{Z\in U} p(Z)>\frac{1}{2}$. Since $X$ is not in the uncovered set, there must be some alternative $W_1$ which pairwise defeats $X$ and all the alternatives in $U$ as well (if such $W_1$ did not exist then $X$ would defeat everyone either directly or in two hops, and thus would be in $S$). If $W_1\in S$, then $\pi(W_1)>\frac{1}{2}$ since it defeats all of $U$, leading us to a contradiction since all alternatives in $S$ cover less than half of the total probability mass. If, on the other hand, $W_1\not\in S$, then by the same argument there must be some alternative $W_2$ which defeats all of $U$, $X$, and $W_1$. We continue in this way until we obtain some alternative $W_k$ which must be in $S$, giving us the desired contradiction.
\end{proof}

This completes the proof of Theorem \ref{thm:uncoveredset_distortion}: by Lemma \ref{lem:prob_mass} we have that $\pi(X)\leq \frac{1}{2}$ for all $X\in S$, by Lemma \ref{lem:outsideOfUncovered} we have that this is true even for $X\not\in S$, and by Theorem \ref{thm:covers} we obtain the desired distortion bound.

\subsection{Median Distortion for 1-Euclidean and Simplex Metrics}
We complete this section by considering special metric spaces. In the case of 1-Euclidean, we are trivially able to obtain the optimal mechanism: selecting the Condorcet winner, since such a winner is guaranteed to exist for the 1-Euclidean setting. By Lemma~\ref{lem:median_condorcet}, we know that the Condorcet winner is guaranteed to be within a factor of 3 of the optimal alternative.


In the $(m-1)$-simplex setting, we will see that almost any alternative is of high median quality. This is due to the fact that the alternatives are very spread out. Unless an alternative has at least $\frac{n}{2}$ agents very close to it, its median cost is guaranteed to be at least $\frac{1}{2}$ in general metric spaces. The following result shows than any mechanism which selects an alternative preferred by more than $\frac{n}{2}$ voters as their top choice (if one exists) will have low median distortion. Thus, for example, plurality is a good mechanism for this setting.

\begin{thm}
\label{thm:median_simplex_dec}
If the $(m-1)$-simplex setting is $\alpha$-decisive, any mechanism that satisfies the majority criterion has median distortion at most $1+\alpha$.
\end{thm}
\begin{proof}
We begin by noting the following useful fact: for any alternative $Y$ with $|Y^*|\leq\frac{n}{2}$, it must be that $\med(Y) \geq \frac{1}{1+\alpha}$. Note that for the case when $n$ is even, $\med(Y)$ refers to the distance of the $(\frac{n}{2}+1)$-th furthest voter from $Y$. This fact is true because at most $\frac{n}{2}$ agents have $Y$ as their first preference, and the remaining agents have $d(i,Y) \geq \frac{1}{1+\alpha}$ by Lemma~\ref{lem:dec_agent_lb} and the fact that the distances between all alternatives equal 1 due to the simplex setting.

If there is no strict majority winner (i.e., all alternatives have at most $\frac{n}{2}$ agents choosing them as their top preference), then the above fact immediately implies the desired bound on median distortion. This is because any alternative $W$ has $\med(W) \leq 1$, since for all $i \in N$, we have that $d(i,W) \leq 1$ due to this being the simplex setting. Therefore, the distortion is always at most $1+\alpha$, since the median cost of any alternative lies between $\frac{1}{1+\alpha}$ and $1$.

Now we consider the case when there is a strict majority winner $W$, i.e., an alternative such that $|W^*| > \frac{n}{2}$. Let $X$ be the optimum alternative (i.e., the one minimizing $\med(X)$). If $X=W$, then the distortion is 1, so assume that $X\neq W$. Then, for every $i\in W^*$, we know that $d(i,W)\leq \alpha\cdot d(i,X)$ (since the distances are $\alpha$-decisive), and that $d(i,X)\leq 1$ (since this is the $(m-1)$-simplex setting). Therefore, it must be that $\med(W)\leq \alpha$. Due to the useful fact proven above, we know that $\med(X) \geq \frac{1}{1+\alpha}$, and the distortion is at most $\alpha(1+\alpha)\leq 1+\alpha$, as desired.
\end{proof}




\section{Conclusion}

We analyzed the distortion of randomized social choice mechanisms in a setting where agent costs form a metric space. In cases where randomized mechanisms are appropriate, such as when the choice will be repeated many times, or when the probability $p(Y)$ can be thought of as the amount of power that candidate $Y$ gets, with the total amount of power summing to 1, then a very small amount of information is necessary to form mechanisms with very small distortion. The randomized mechanisms we consider, such as  randomized dictatorship and proportional to squares, require only the first preference of each agent (and do not require the agent costs to be known) to achieve distortion better than any known deterministic mechanism, despite the fact that the best deterministic mechanisms require the full preference profile. Thus, randomized mechanisms can perform better with less information. We also considered special cases of metrics, such as when agents have a strong preference towards their first choice over the second, when agent preferences are $1$-Euclidean, and when all alternatives are completely dissimilar, and we were able to achieve better distortion bounds using randomized mechanisms. This was true even for the more ``egalitarian" median objective, in which we were able to provide a probability distribution based on the majority graph with distortion better than any deterministic mechanism.

Some open questions still remain. While we were able to show that proportional to squares is an optimal mechanism for $m=2$ alternatives in the sum setting, our best known mechanism for arbitrary $m$ is randomized dictatorship, which has a distortion arbitrarily close to 3 in the worst case. We suspect there may exist a generalization of proportional to squares that is able to achieve a distortion of 2, but it is likely significantly more complex and may require the full preference profile instead of agents' top preferences.

More generally, this paper studies how well algorithms and mechanisms which only have access to limited ordinal information, instead of the ground truth, can compete with truly omniscient algorithms. These questions are larger than just social choice, and apply to other settings as well, for example matchings, as described in the Introduction. At least for metric settings, it seems that knowing only ordinal information is often enough; there is no need to elicit complex numerical information. Looking at other utility structures in addition to metric spaces would also make sense, such as specific metric spaces which model particular applications (e.g., doubling metrics), or algorithms which know limited numerical information (e.g., the order of magnitude of the agent utilities), but must compete with omniscient mechanisms. Finally, it would be interesting to analyze the normative properties of mechanisms with small distortion: it may be possible to characterize the entire space of mechanisms which have, e.g., distortion at most 3 and obey certain desirable axiomatic properties.

\section*{Acknowledgements}
We owe great thanks to Edith Elkind for many interesting and informative discussions on the topic of social choice. This work was supported in part by NSF awards CCF-1101495, CNS-1218374, and CCF-1527497.

\end{document}